\documentclass[a4paper]{llncs}

\usepackage{amsmath}
\usepackage{amssymb}
\usepackage{graphicx}
\usepackage{url}
\usepackage{verbatim}
\usepackage[noblocks]{authblk}
\usepackage{array}
\usepackage{times,helvet,courier}
\usepackage{multirow}
\usepackage{multicol}
\usepackage{epsfig}

\usepackage{a4,a4wide}

\newcommand{\rg}[3]{{#1}{#2}\ldots{#2}{#3}}
\newcommand{\set}[1]{\{#1\}}
\newcommand{\enum}[2]{#1,\ldots,#2}
\newcommand{\eset}[2]{\set{\enum{#1}{#2}}}
\newcommand{\sel}[2]{\set{#1\mid #2}}
\newcommand{\pair}[2]{\langle{#1},{#2}\rangle}
\newcommand{\union}{\cup}
\newcommand{\dunion}{\sqcup}
\newcommand{\Union}{\bigcup}
\newcommand{\isect}{\cap}
\newcommand{\compl}[1]{\overline{#1}}
\newcommand{\limpl}{\rightarrow}
\newcommand{\lequiv}{\leftrightarrow}
\newcommand{\Lor}{\bigvee}
\newcommand{\Land}{\bigwedge}

\newcommand{\choice}[1]{\{#1\}}
\newcommand{\limit}[2]{{#1}\{{#2}\}}
\newcommand{\IF}{\leftarrow}
\DeclareMathSymbol{\naf}{\mathord}{symbols}{"18}
\newcommand{\END}{.}
\newcommand{\SEP}{;\ }

\newcommand{\hd}[1]{\mathrm{hd}(#1)}
\newcommand{\bd}[1]{\mathrm{bd}(#1)}
\newcommand{\bdplus}[1]{\mathrm{bd}^+(#1)}
\newcommand{\bdminus}[1]{\mathrm{bd}^-(#1)}
\newcommand{\bdatom}[1]{\mathrm{bd}_{#1}}

\newcommand{\width}[1]{\mathrm{w}(#1)}
\newcommand{\bin}[1]{\overline{#1}}
\newcommand{\bvadd}[3]{#1+_{#2}#3}
\newcommand{\bveq}[3]{#1=_{#2}#3}
\newcommand{\bvlt}[3]{#1<_{#2}#3}
\newcommand{\bvge}[3]{#1\geq_{#2}#3}

\newcommand{\val}{\tau}
\newcommand{\hb}[1]{\mathrm{At}(#1)}
\newcommand{\hu}[1]{\mathrm{FC}(#1)}
\newcommand{\rext}[1]{\mathrm{ext}_{#1}}
\newcommand{\rint}[1]{\mathrm{int}_{#1}}

\newcommand{\GLred}[2]{#1^{#2}}
\newcommand{\lm}[1]{\mathrm{LM}(#1)}
\newcommand{\modelset}[2]{\mathrm{#1}(#2)}
\newcommand{\tr}[2]{\mathrm{#1}(#2)}
\newcommand{\defof}[2]{\mathrm{Def}_{#1}(#2)}
\newcommand{\edefof}[2]{\mathrm{Ext}_{#1}(#2)}
\newcommand{\idefof}[2]{\mathrm{Int}_{#1}(#2)}
\newcommand{\DG}[1]{\mathrm{DG}^+(#1)}
\newcommand{\SCC}[1]{\mathrm{SCC}(#1)}
\newcommand{\comp}[1]{\mathrm{Comp}(#1)}

\newcommand{\ord}[1]{{#1}^{\mathrm{th}}}
\newcommand{\system}[1]{\textsc{#1}}
\newcommand{\eofex}{\mbox{}\nobreak\hfill\hspace{0.5em}$\blacksquare$}
\newcommand{\hl}[1]{\textbf{#1}}

\begin{document}

\title{Translating Answer-Set Programs into Bit-Vector Logic%
\thanks{This paper appears in the Proceedings of the 19th International
Conference on Applications of Declarative Programming and Knowledge
Management (INAP 2011).}}

\author{Mai Nguyen \and Tomi Janhunen \and Ilkka Niemel\"a}

\institute{%
Aalto University School of Science \\
Department of Information and Computer Science \\
\email{\{Mai.Nguyen,Tomi.Janhunen,Ilkka.Niemela\}@aalto.fi}}

\maketitle

\begin{abstract}
Answer set programming (ASP) is a paradigm for declarative problem
solving where problems are first formalized as rule sets, i.e.,
answer-set programs, in a uniform way and then solved by computing
answer sets for programs. The satisfiability modulo theories (SMT)
framework follows a similar modelling philosophy but the syntax is
based on extensions of propositional logic rather than rules.
Quite recently, a translation from answer-set programs into difference
logic was provided---enabling the use of particular SMT solvers for
the computation of answer sets.
In this paper, the translation is revised for another SMT fragment,
namely that based on fixed-width bit-vector theories. Thus, even
further SMT solvers can be harnessed for the task of computing answer
sets. The results of a preliminary experimental comparison are also
reported. They suggest a level of performance which is similar to
that achieved via difference logic.
\end{abstract}

%------------------------------------------------------------------------------

\section{Introduction}

% Intro to ASP

Answer set programming (ASP) is a rule-based approach to declarative
problem solving
\cite{DBLP:journals/ai/GelfondL02a,m-t-99,DBLP:journals/amai/Niemela99}.
The idea is to first formalize a given problem as a set of rules also
called an \emph{answer-set program} so that the answer sets of the
program correspond to the solution of the problem. Such problem
descriptions are typically devised in a \emph{uniform} way which
distinguishes general principles and constraints of the problem in
question from any instance-specific data. To this end, term variables
are deployed for the sake of compact representation of rules.
Solutions themselves can then be found out by \emph{grounding} the
rules of the answer-set program, and by computing answer sets
for the resulting ground program using an answer set solver. 
State-of-the-art answer set solvers are already very efficient
search engines
\cite{DBLP:conf/lpnmr/CalimeriIRABCCFFLMMPPRSSTV11,%
DBLP:conf/lpnmr/DeneckerVBGT09}
and have a wide range of industrial applications.

% Intro to SMT

The satisfiability modulo theories (SMT) framework
\cite{DBLP:series/faia/BarrettSST09}
follows a similar modelling philosophy but the syntax is based on
extensions of propositional logic rather than rules with term
variables.  The SMT framework enriches traditional satisfiability
(SAT) checking \cite{DBLP:series/faia/2009-185} in terms of background
theories which are selected amongst a number of
alternatives.\footnote{\url{http://combination.cs.uiowa.edu/smtlib/}}
Parallel to propositional atoms, also \emph{theory atoms} involving
non-Boolean variables%
\footnote{%
However, variables in SMT are syntactically represented by (functional)
constants having a free interpretation over a specific domain such as
integers or reals.}
can be used as references to potentially infinite domains. Theory
atoms are typically used to express various constraints such as linear
constraints, difference constraints, etc., and they enable very
concise representations of certain problem domains for which plain
Boolean logic would be more verbose or insufficient
in the first place.

% Translations from ASP to SMT

As regards the relationship of ASP and SMT, it was quite recently shown
\cite{DBLP:conf/lpnmr/JanhunenNS09,DBLP:journals/amai/Niemela08}
that answer-set programs can be efficiently translated into a simple
SMT fragment, namely \emph{difference logic} (DL)
\cite{DBLP:conf/cav/NieuwenhuisO05}.
This fragment is based on theory atoms of the form
$x-y\leq k$
formalizing an upper bound $k$ on the \emph{difference} of two
integer-domain variables $x$ and $y$. Although the required
transformation is linear, it is not reasonable to expect that such
theories are directly written by humans in order to express
the essentials of ASP in SMT.
The translations from
\cite{DBLP:conf/lpnmr/JanhunenNS09,DBLP:journals/amai/Niemela08}
and their implementation called \system{lp2diff}%
\footnote{\url{http://www.tcs.hut.fi/Software/lp2diff/}}
enable the use of particular SMT solvers for the computation of answer
sets. Our experimental results \cite{DBLP:conf/lpnmr/JanhunenNS09}
indicate that the performance obtained in this way is surprisingly
close to that of state-of-the-art answer set solvers. The results
of the third ASP competition
\cite{DBLP:conf/lpnmr/CalimeriIRABCCFFLMMPPRSSTV11},
however, suggest that the performance gap has grown since the previous
competition.
To address this trend, our current and future agendas include
a number of points:

\begin{itemize}
\item
We gradually increase the number of supported SMT fragments
which enables the use of further SMT solvers for the task of
computing answer sets.

\item
We continue the development of new translation
techniques from ASP to SMT.

\item
We submit ASP-based benchmark sets to future SMT competitions
(SMT-COMPs) to foster the efficiency of SMT solvers on problems
that are relevant for ASP.

\item
We develop new integrated languages that combine features of ASP and
SMT, and aim at implementations via translation into pure SMT as
initiated in \cite{JLN11:gttv}.
\end{itemize}

% Goals of this paper

This paper contributes to the first item by devising a translation
from answer-set programs into theories of bit-vector logic. There is a
great interest to develop efficient solvers for this particular SMT
fragment due to its industrial relevance. In view of the
second item, we generalize an existing translation from
\cite{DBLP:conf/lpnmr/JanhunenNS09} to the case of bit-vector logic.
Using an implementation of the new translation, viz.~\system{lp2bv},
new benchmark classes can be created to support the third item on our
agenda. Finally, the translation also creates new potential for
language integration. In the long run, rule-based languages and, in
particular, the modern grounders exploited in ASP can provide valuable
machinery for the generation of SMT theories in analogy to answer-set
programs: The \emph{source code} of an SMT theory can be compacted
using rules and term variables \cite{JLN11:gttv} and specified in a
uniform way which is independent of any concrete problem
instances. Analogous approaches
\cite{DBLP:conf/lpnmr/Balduccini11,%
DBLP:conf/iclp/GebserOS09,DBLP:conf/flops/MellarkodG08}
combine ASP and constraint programming techniques without a translation.

% Structure

The rest of this paper is organized as follows.
First, the basic definitions and concepts of answer-set programs and
fixed-width bit-vector logic are briefly reviewed in Section
\ref{section:preliminaries}.
The new translation from answer-set programs into bit-vector theories is
then devised in Section \ref{section:translation}.
The extended rule types of \system{smodels} compatible systems are
addressed in Section \ref{section:native-extensions}. Such extensions
can be covered either by native translations into bit-vector logic or
translations into normal programs.
As part of this research, we carried out a number of experiments using
benchmarks from the second ASP competition
\cite{DBLP:conf/lpnmr/DeneckerVBGT09}
and two state-of-the-art SMT solvers, viz.~\system{boolector} and
\system{z3}.  The results of the experiments are reported in Section
\ref{section:experiments}.  Finally, we conclude this paper in Section
\ref{section:conclusion} in terms of discussions of results and future work.

%------------------------------------------------------------------------------

\section{Preliminaries}
\label{section:preliminaries}

The goal of this section is to briefly review the source and
target formalisms for the new translation devised in the
sequel. First, in Section \ref{section:normal-programs}, we recall
normal logic programs subject to answer set semantics and the main
notions exploited in their translation. A formal account of bit-vector
logic follows in Section \ref{section:bit-vector-logic}.

\subsection{Normal Logic Programs}
\label{section:normal-programs}

As usual, we define a \emph{normal logic program} $P$ 
as a finite set of \emph{rules} of the form
\begin{equation}
\label{eq:rule}
a \IF \enum{b_1}{b_n},\enum{\naf c_1}{\naf c_m}
\end{equation}
where $a$, $\enum{b_1}{b_n}$, and $\enum{c_1}{c_m}$ are propositional
atoms and $\naf$ denotes \emph{default negation}.  The \emph{head} of
a rule $r$ of the form (\ref{eq:rule}) is $\hd{r}=a$ whereas the part
after the symbol $\IF$ forms the \emph{body} of $r$, denoted by $\bd{r}$.
The body $\bd{r}$ consists of the positive part
$\bdplus{r} = \eset{b_1}{b_n}$
and the negative part
$\bdminus{r} = \eset{c_1}{c_m}$
so that
$\bd{r} = \bdplus{r} \union \sel{\naf c}{c \in \bdminus{r}}$.
Intuitively, a rule $r$ of the form (\ref{eq:rule}) appearing in a
program $P$ is used as follows: the head $\hd{r}$ can be inferred by
$r$ if the \emph{positive body atoms} in $\bdplus{r}$ are inferable by
the other rules of $P$, but not the \emph{negative body atoms} in
$\bdminus{r}$. The positive part of the rule, $r^+$ is defined as
$\hd{r}\IF\bdplus{r}$. A normal logic program is called
\emph{positive} if $r=r^+$ holds for every rule $r \in P$.

\paragraph{Semantics}
To define the semantics of a normal program $P$, we let $\hb{P}$ stand
for the set of atoms that appear in $P$.  An \emph{interpretation} of
$P$ is any subset $I \subseteq \hb{P}$ such that for an atom
$a\in\hb{P}$, $a$ is \textit{true} in $I$, denoted $I\models a$, iff
$a\in I$. For any negative literal $\naf c$, 
$I \models \naf c$ iff $I\not\models c$ iff $c\not\in I$.
A rule $r$ is satisfied in $I$, denoted $I \models r$, iff
$I\models\bd{r}$ implies $I \models \hd{r}$.  An interpretation $I$ is
a \emph{classical model} of $P$, denoted $I \models P$, iff, $I\models r$
holds for every $r \in P$. A model $M \models P$ is a \emph{minimal model}
of $P$ iff there is no $M' \models P$ such that $M'\subset M$. Each
positive normal program $P$ has a unique minimal model, i.e., the
\emph{least model} of $P$ denoted by $\modelset{LM}{P}$ in the sequel.
The least model semantics can be extended for an arbitrary normal
program $P$ by \emph{reducing} $P$ into a positive program
$\GLred{P}{M}=\sel{r^+}{r \in P\text{ and }M\isect\bdminus{r} = \emptyset}$
with respect to $M\subseteq\hb{P}$.
Then \emph{answer sets}, also known as \emph{stable models}
\cite{DBLP:conf/iclp/GelfondL88}, can be defined.

\begin{definition}[Gelfond and Lifschitz \cite{DBLP:conf/iclp/GelfondL88}]
\label{def:answer-set}
An interpretation $M\subseteq \hb{P}$ is an \emph{answer set} of a normal
program $P$ iff $M = \lm{\GLred{P}{M}}$.
\end{definition}

\begin{example}
\label{ex:answer-sets}
Consider a normal program $P$ \cite{DBLP:conf/lpnmr/JanhunenNS09}
consisting of the following six rules:
\[
\newcommand{\ws}{\hspace{3em}}
\begin{array}{l@{\ws}l@{\ws}l}
a\IF b, c \END & a \IF d \END  &  b\IF a, \naf d \END \\
b \IF a, \naf c \END & c \IF \naf d \END &  d \IF \naf c \END
\end{array}
\]
The answer sets of $P$ are $M_1=\set{a,b,d}$ and $M_2=\set{c}$.
To verify the latter, we note that
$\GLred{P}{M_2}=\set{%
a\IF b,c\SEP
b\IF a\SEP
c\IF\SEP
a\IF d}$
for which $\lm{\GLred{P}{M_2}}=\set{c}$. On the other hand, we have
$\GLred{P}{M_3}=\GLred{P}{M_2}$ for $M_3=\set{a,b,c}$ so
that $M_3\not\in\modelset{AS}{P}$.
\eofex
\end{example}

The number of answer sets possessed by a normal program $P$ can vary
in general. The set of answer sets of a normal program $P$ is denoted
by $\modelset{AS}{P}$. Next we present some concepts and results that
are relevant in order to capture answer sets in terms of propositional
logic and its extensions in the SMT framework.

\paragraph{Completion}
Given a normal program $P$ and an atom $a\in\hb{P}$, the \emph{definition}
of $a$ in $P$ is the set of rules
$\defof{P}{a}=\sel{r\in P}{\hd{r}=a}$.  The \emph{completion} of a
normal program $P$, denoted by $\comp{P}$, is a propositional theory
\cite{DBLP:conf/adbt/Clark77} which contains
\begin{equation}
\label{eq:completion}
a \lequiv
\Lor_{r\in\defof{P}{a}}
  \bigl(\Land_{b\in\bdplus{r}}b~~~\land \Land_{c\in\bdminus{r}}\neg c \bigr)
\end{equation}
for each atom $a \in \hb{P}$. Given a propositional theory $T$ and
its signature $\hb{T}$, the semantics of $T$ is determined by
$\modelset{CM}{T}=\sel{M\subseteq\hb{T}}{M \models T}$.
It is possible to relate $\modelset{CM}{\comp{P}}$ with
the models of a normal program $P$ by distinguishing
\emph{supported models} \cite{DBLP:books/mk/minker88/AptBW88}
for $P$. A model $M\models P$ is a supported model of $P$ iff for
every atom $a\in M$ there is a rule $r \in P$ such that $\hd{r} = a$
and $M \models \bd{r}$. In general, the set of supported models
$\modelset{SuppM}{P}$ of a normal program $P$ coincides with
$\modelset{CM}{\comp{P}}$. It can be shown
\cite{DBLP:journals/tcs/MarekS92} that stable models are also
supported models but not necessarily vice versa.  This means that in
order to capture $\modelset{AS}{P}$ using $\comp{P}$, the latter has to
be extended in terms of additional constraints as done, e.g., in
\cite{Janhunen06:jancl,DBLP:conf/lpnmr/JanhunenNS09}.

\begin{example}
\label{ex:completion}
For the program $P$ of Example \ref{ex:answer-sets}, the theory
$\comp{P}$ has formulas
$a\lequiv (b\land c)\lor d$,
$b\lequiv (a\land\neg d)\lor(a\land\neg c)$,
$c\lequiv\neg d$, and
$d\lequiv\neg c$.
The models of $\comp{P}$, i.e., its supported models, are
$M_1=\set{a,b,d}$,
$M_2=\set{c}$, and
$M_3=\set{a,b,c}$.
\eofex
\end{example}

\paragraph{Dependency Graphs}
The \emph{positive dependency graph} of a normal program $P$, denoted
by $\DG{P}$, is a pair $\pair{\hb{P}}{\leq}$ where $b \leq a$ holds iff
there is a rule $r \in P$ such that $\hd{r} = a$ and $b \in \bdplus{r}$.
Let $\leq^*$ denote the \emph{reflexive} and \emph{transitive} closure
of $\leq$.
A \emph{strongly connected component} (SCC) of $\DG{P}$ is a maximal
non-empty subset $S\subseteq\hb{P}$ such that $a \leq^* b$ and
$b \leq^* a$ hold for each $a,b \in S$.
The set of defining rules is generalized for an SCC $S$ by
$\defof{P}{S}=\Union_{a\in S}\defof{P}{a}$.
This set can be naturally partitioned into sets
$\edefof{P}{S}=\sel{r\in\defof{P}{S}}{\bdplus{r}\isect S=\emptyset}$
and
$\idefof{P}{S}=\sel{r\in\defof{P}{S}}{\bdplus{r}\isect S\neq\emptyset}$
of \emph{external} and \emph{internal} rules associated with $S$,
respectively. Thus, $\defof{P}{S}=\edefof{P}{S}\dunion\idefof{P}{S}$
holds in general.

\begin{example}
In the case of the program $P$ from Example \ref{ex:answer-sets}, the
SCCs of $\DG{P}$ are $S_1=\set{a,b}$, $S_2=\set{c}$, and $S_3=\set{d}$.
For $S_1$, we have
$\edefof{P}{S_1}=\set{a\IF d}$.
\eofex
\end{example}

%------------------------------------------------------------------------------

\subsection{Bit-Vector Logic}
\label{section:bit-vector-logic}

\emph{Fixed-width bit-vector} theories have been introduced for
high-level reasoning about digital circuitry and computer programs in
the SMT framework
\cite{DBLP:journals/expert/BeckertHHSGRTBR06,SMT-LIB}.
Such theories are expressed in an extension of propositional logic
where atomic formulas speak about bit vectors in terms of a rich
variety of operators.

\paragraph{Syntax}
As usual in the context of SMT, variables are realized as 
constants that have a free interpretation over a particular domain (such as
integers or reals)\footnote{We use typically symbols $x,y,z$ to denote
 such free (functional) constants and symbols $a,b,c$ to denote
 propositional atoms.}.  
In the case of fixed-width bit-vector theories, this
means that each constant symbol $x$ represents a vector $x[1\ldots m]$
of bits of particular width $m$, denoted by $\width{x}$ in the sequel.
Such vectors enable a more compact representation of structures like
registers and often allow more efficient reasoning about them.
A special notation $\bin{n}$ is introduced to denote a bit vector that
equals to $n$, i.e., $\bin{n}$ provides a binary representation of
$n$. We assume that the actual width $m\geq\log_2(n+1)$ is determined
by the context where the notation $\bin{n}$ is used.
For the purposes of this paper, the most interesting arithmetic
operator for combining bit vectors is the addition of two $m$-bit
vectors, denoted by the parameterized function symbol $\bvadd{}{m}{}$
in an infix notation. The resulting vector is also $m$-bit which can
lead to an overflow if the sum exceeds $2^m-1$. Moreover, we use
Boolean operators $\bveq{}{m}{}$ and $\bvlt{}{m}{}$ with the usual
meanings for comparing the values of two $m$-bit vectors. Thus,
assuming that $x$ and $y$ are $m$-bit free constants, we may write
atomic formulas like $\bveq{x}{m}{y}$ and $\bvlt{x}{m}{y}$ in order
to compare the $m$-bit values of $x$ and $y$.
In addition to syntactic elements mentioned so far, we can use the
primitives of propositional logic to build more complex
\emph{well-formed formulas} of bit-vector logic. The syntax defined
for the SMT library contains further primitives which are skipped in
this paper. A theory $T$ in bit-vector logic is a set of well-formed
bit-vector formulas as illustrated by the following example.

\begin{example}
\label{ex:bit-vector-theory}
Consider a system of two processes, say A and B, and a theory
$T=\set{a\limpl(\bvlt{x}{2}{y}),\ b\limpl(\bvlt{y}{2}{x})}$
formalizing a scheduling policy for them.
The intuitive reading of $a$ (resp.~$b$) is that process A (resp.~B)
is scheduled with a higher priority and, thus, should start
earlier. The constants $x$ and $y$ denote the respective starting
times of A and B. Thus, e.g., $\bvlt{x}{2}{y}$ means that process A
starts before process B.
\eofex
\end{example}

\paragraph{Semantics}
Given a bit-vector theory $T$, we write $\hb{T}$ and $\hu{T}$ for the
sets of propositional atoms and free constants, respectively, appearing in
$T$. To determine the semantics of $T$, we define
\emph{interpretations} for $T$ as pairs $\pair{I}{\val}$ where
$I\subseteq\hb{T}$ is a standard propositional interpretation and
$\val$ is a partial function that maps a free constant $x\in\hu{T}$ and an
index $1\leq i\leq\width{x}$ to the set of bits $\set{0,1}$.
Given $\val$, a constant $x\in\hu{T}$ is mapped onto
$\val(x)=\sum_{i=1}^{\width{x}} (\val(x,i)\cdot 2^{\width{x}-i})$
and, in particular, $\val(\bin{n})=n$ for any $n$.
To cover any \emph{well-formed terms}%
\footnote{The constants and operators appearing in a well-formed
term $t$ are based on a fixed width $m$. Moreover, the width
$\width{x}$ of each constant $x\in\hu{T}$ must be the same
throughout $T$.}
$t_1$ and $t_2$ involving $\bvadd{}{m}{}$ and $m$-bit constants from
$\hu{T}$, we define
$\val(\bvadd{t_1}{m}{t_2})=\val(t_1)+\val(t_2)\mod 2^m$
and $\width{\bvadd{t_1}{m}{t_2}}=m$.
Hence, the value $\val(t)$ can be determined for any well-formed term
$t$ which enables the evaluation of more complex formulas as
formalized below.

\begin{definition}
\label{def:bit-vector-model}
Let $T$ be a bit-vector theory, $a\in\hb{T}$ a propositional atom,
$t_1$ and $t_2$ well-formed terms over $\hu{T}$ such that
$\width{t_1}=\width{t_2}$, and $\phi$ and $\psi$ well-formed
formulas. Given an interpretation $\pair{I}{\val}$
for the theory $T$, we define
\begin{enumerate}
\item
$\pair{I}{\val}\models a$ $\iff$ $a\in I$,

\item
$\pair{I}{\val}\models\bveq{t_1}{m}{t_2}$ $\iff$ $\val(t_1)=\val(t_2)$,

\item
$\pair{I}{\val}\models\bvlt{t_1}{m}{t_2}$ $\iff$ $\val(t_1)<\val(t_2)$,

\item
$\pair{I}{\val}\models\neg\phi$ $\iff$ $\pair{I}{\val}\not\models\phi$,

\item
$\pair{I}{\val}\models\phi\lor\psi$ $\iff$
$\pair{I}{\val}\models\phi$ or $\pair{I}{\val}\models\psi$,

\item
$\pair{I}{\val}\models\phi\limpl\psi$ $\iff$
$\pair{I}{\val}\not\models\phi$ or $\pair{I}{\val}\models\psi$, and

\item
$\pair{I}{\val}\models\phi\lequiv\psi$ $\iff$
$\pair{I}{\val}\models\phi$ if and only if $\pair{I}{\val}\models\psi$.
\end{enumerate}
The interpretation $\pair{I}{\val}$ is a model of $T$, i.e.,
$\pair{I}{\val} \models T$, iff $\pair{I}{\val}\models\phi$
for all $\phi\in T$.
\end{definition}

It is clear by Definition \ref{def:bit-vector-model} that pure
propositional theories $T$ are treated classically, i.e.,
$\pair{I}{\val}\models T$ iff $I\models T$ in the sense of
propositional logic.
As regards the theory $T$ from Example \ref{ex:bit-vector-theory}, we
have the sets of symbols $\hb{T}=\set{a,b}$ and $\hu{T}=\set{x,y}$.
Furthermore, we observe that there is no model of $T$ of the form
$\pair{\set{a,b}}{\val}$ because it is impossible to satisfy
$\bvlt{x}{2}{y}$ and $\bvlt{y}{2}{x}$ simultaneously using any partial
function $\val$.
On the other hand, there are $6$ models of the form
$\pair{\set{a}}{\val}$ because $\bvlt{x}{2}{y}$ can be satisfied in
$3+2+1=6$ ways by picking different values for the 2-bit vectors
$x$ and $y$.

%------------------------------------------------------------------------------

\section{Translation}
\label{section:translation}

In this section, we present a translation of a logic program $P$ into a
bit-vector theory $\tr{BV}{P}$ that is similar to an existing
translation \cite{DBLP:conf/lpnmr/JanhunenNS09} into difference logic.
As its predecessor, the translation $\tr{BV}{P}$ consists of two
parts.  Clark's completion \cite{DBLP:conf/adbt/Clark77}, denoted by
$\tr{CC}{P}$, forms the first part of $\tr{BV}{P}$. The second part,
i.e., $\tr{R}{P}$, is based on
\emph{ranking constraints} from \cite{DBLP:journals/amai/Niemela08}
so that $\tr{BV}{P}=\tr{CC}{P}\union\tr{R}{P}$.
Intuitively, the idea is that the completion $\tr{CC}{P}$ captures
\emph{supported models} of $P$ \cite{DBLP:books/mk/minker88/AptBW88}
and the further formulas in $\tr{R}{P}$ exclude the non-stable ones so
that any classical model of $\tr{BV}{P}$ corresponds to a stable model
of $P$.

The completion $\modelset{CC}{P}$ is formed for each atom $a\in
\hb{P}$ on the basis of (\ref{eq:completion}):
\begin{enumerate}
\item
If $\defof{P}{a}=\emptyset$, the formula $\neg a$ is included
to capture the corresponding empty disjunction in (\ref{eq:completion}).

\item
If there is $r\in\defof{P}{a}$ such that $\bd{r}=\emptyset$, then
one of the disjuncts in (\ref{eq:completion}) is trivially true and
the formula $a$ can be used as such to capture the definition of $a$.

\item
If $\defof{P}{a}=\set{r}$ for a rule $r\in P$ with $n+m>0$, then
we simplify (\ref{eq:completion}) to a formula of the form
\begin{equation}
a\lequiv\Land_{b\in\bdplus{r}}b~~~\land \Land_{c\in\bdminus{r}}\neg c.
\end{equation}

\item
Otherwise, the set $\defof{P}{a}$ contains at least two rules
(\ref{eq:rule}) with $n+m>0$ and
\begin{equation}
\label{eq:completion-disjunction}
a \lequiv \Lor_{r\in\defof{P}{a}}\bdatom{r}
\end{equation}
is introduced using a new atom $\bdatom{r}$
for each $r\in\defof{P}{a}$ together with a formula
\begin{equation}
\label{eq:completion-definition}
\bdatom{r}
\lequiv
\Land_{b\in\bdplus{r}}b~~~\land\Land_{c\in\bdminus{r}}\neg c.
\end{equation}
\end{enumerate}
The rest of the translation exploits the SCCs of the positive
dependency graph of $P$ that was defined in Section
\ref{section:normal-programs}.
The motivation is to limit the scope of ranking constraints which
favors the length of the resulting translation. In particular,
singleton components $\modelset{SCC}{a}=\set{a}$ require no special
treatment if \emph{tautological} rules with $a\in\eset{b_1}{b_n}$ in
(\ref{eq:rule}) have been removed. Plain completion
(\ref{eq:completion}) is sufficient for atoms involved in such
components.
However, for each atom $a\in\hb{P}$ having a non-trivial component
$\SCC{a}$ in $\DG{P}$ such that $|\modelset{SCC}{a}|>1$, two new atoms
$\rext{a}$ and $\rint{a}$ are introduced to formalize the
\emph{external} and \emph{internal} support for $a$,
respectively. These atoms are defined in terms of equivalences
\begin{eqnarray}
\label{eq:external-support}
\rext{a}\lequiv\Lor_{r\in\edefof{P}{a}}\bdatom{r} \\
\label{eq:internal-support}
\rint{a}\lequiv
  \Lor_{r\in\idefof{P}{a}}\bigl[
    \bdatom{r}\land
       \Land_{b\in\bdplus{r}\isect\SCC{a}}(\bvlt{x_b}{m}{x_a})\bigr]
\end{eqnarray}
where $x_a$ and $x_b$ are bit vectors of width
$m=\lceil\log_2(|\modelset{SCC}{a}|+1)\rceil$
introduced for all atoms involved in $\SCC{a}$.  The formulas
(\ref{eq:external-support}) and (\ref{eq:internal-support}) are called
\emph{weak} ranking constraints and they are accompanied by
\begin{eqnarray}
\label{eq:support}
a\limpl\rext{a}\lor\rint{a}, \\
\label{eq:exclusion-for-suport}
\neg\rext{a}\lor\neg\rint{a}.
\end{eqnarray}
Moreover, when $\edefof{P}{a}\neq\emptyset$ and the atom $a$ happens
to gain external support from these rules, the value of $x_a$ is fixed
to $0$ by including the formula
\begin{equation}
\label{eq:clear}
\rext{a}\limpl(\bveq{x_a}{m}{\bin{0}}).
\end{equation}

\begin{example}
Recall the program $P$ from Example \ref{ex:answer-sets}. 
The completion $\tr{CC}{P}$ is:
\begin{center}
$
\begin{array}{rcl@{\quad}rcl@{\quad}rcl}
a & \lequiv & \bdatom{1} \lor \bdatom{2} \END&
\bdatom{1} & \lequiv & b \land c \END &
\bdatom{2} & \lequiv & d \END\\
b & \lequiv & \bdatom{3} \lor \bdatom{4} \END& 
\bdatom{3} & \lequiv & a \land \neg d \END &
\bdatom{4} & \lequiv & a \land \neg c \END \\
c & \lequiv & \neg d \END \\
d & \lequiv & \neg c \END
\end{array}
$
\end{center}
Since $P$ has only one non-trivial SCC, i.e., the component
$\SCC{a}=\SCC{b}=\set{a,b}$, the weak ranking constraints resulting in
$\modelset{R}{P}$ are
\begin{center}
$
\begin{array}{rcl@{\quad}rcl}
\rext{a} & \lequiv & \bdatom{2} \END &
\rint{a} & \lequiv & \bdatom{1}\land (\bvlt{x_b}{2}{x_a}) \END \\
\rext{b} & \lequiv & \perp \END \\
\rint{b} & \lequiv &
\multicolumn{4}{l}{
  [\bdatom{3}\land(\bvlt{x_a}{2}{x_b})]\lor
  [\bdatom{4}\land(\bvlt{x_a}{2}{x_b})]\END}
\end{array}
$
\end{center}
In addition to these, the formulas
\begin{center}
$
\begin{array}{rcl@{\quad}c@{\quad}rcl}
a & \limpl & \rext{a} \lor \rint{a}\END &
\neg \rext{a} \lor \neg \rint{a}\END &
\rext{a} & \limpl & (\bveq{x_a}{2}\bin{0})\END\\
b & \limpl & \rext{b} \lor \rint{b}\END &
\neg \rext{b} \lor \neg \rint{b}\END &&&
\end{array}
$
\end{center}
are also included in $\tr{R}{P}$.
\eofex
\end{example}

Weak ranking constraints are sufficient whenever the goal is to
compute only one answer set, or to check the existence of answer sets.
However, they do not guarantee a one-to-one correspondence between the
elements of $\modelset{AS}{P}$ and the set of models obtained for the
translation $\tr{BV}{P}$. To address this discrepancy, and to
potentially make the computation of all answer sets or counting the
number of answer sets more effective, \emph{strong} ranking
constraints can be imported from
\cite{DBLP:conf/lpnmr/JanhunenNS09} 
as well. Actually, there are two mutually compatible variants of
strong ranking constraints:
\begin{eqnarray}
\label{eq:strong-local}
\bdatom{r}\limpl
   \Lor_{b\in\bdplus{r}\isect\SCC{a}}\neg(\bvlt{\bvadd{x_b}{m}{\bin{1}}}{m}{x_a})
\\
\label{eq:strong-global}
\rint{a}\limpl
\Lor_{r\in\idefof{P}{a}}
  [\bdatom{r}\land
   \Lor_{b\in\bdplus{r}\isect\SCC{a}}
    (\bveq{x_a}{m}{\bvadd{x_b}{m}{\bin{1}}})].
\end{eqnarray}
The \emph{local} strong ranking constraint (\ref{eq:strong-local}) is
introduced for each $r\in\idefof{P}{a}$. It is worth pointing out that
the condition $\neg(\bvlt{\bvadd{x_b}{m}{\bin{1}}}{m}{x_a})$ is
equivalent to $\bvge{\bvadd{x_b}{m}{\bin{1}}}{m}{x_a}$.~\footnote{However, the form in (\ref{eq:strong-local}) is used in our
  implementation, since $\bvadd{}{m}{}$ and $\bvlt{}{m}{}$ are amongst
  the base operators of the \system{boolector} system.}
On the other hand, the \emph{global} variant (\ref{eq:strong-global})
covers the internal support of $a$ entirely.
Finally, in order to prune copies of models of the translation that
would correspond to the exactly same answer set of the original
program, a formula
\begin{equation}
\label{eq:clear-false}
\neg a\limpl(\bveq{x_a}{m}{\bin{0}})
\end{equation}
is included for every atom $a$ involved in a non-trivial SCC. We write
$\tr{R^{l}}{P}$ and $\tr{R^{g}}{P}$ for the respective extensions of
$\tr{R}{P}$ with local/global strong ranking constraints, and
$\tr{R^{lg}}{P}$ obtained using both. Similar conventions are applied
to $\tr{BV}{P}$ to distinguish four variants in total. The correctness
of these translations is addressed next.

\begin{theorem}
Let $P$ be a normal program and $\tr{BV}{P}$ its bit-vector translation.
\begin{enumerate}
\item
If $S$ is an answer set of $P$, then there is a model
$\pair{M}{\val}$ of $\tr{BV}{P}$ such that $S=M\isect\hb{P}$.

\item
If $\pair{M}{\val}$ is a model of $\tr{BV}{P}$, then
$S=M\isect\hb{P}$ is an answer set of $P$.
\end{enumerate}
\end{theorem}

\begin{proof}
To establish the correspondence of answer sets and models as
formalized above, we appeal to the analogous property of the
translation of $P$ into difference logic (DL), denoted here by
$\tr{DL}{P}$. In DL, theory atoms $x\leq y+k$ constrain the
difference of two integer variables $x$ and $y$. Models can be
represented as pairs $\pair{I}{\val}$ where $I$ is a propositional
interpretation and $\val$ maps constants of theory atoms to integers
so that
$\pair{I}{\val}\models x\leq y+k$ $\iff$ $\val(x)\leq\val(y)+k$.
The rest is analogous to Definition \ref{def:bit-vector-model}.

($\implies$)
Suppose that $S$ is an answer set of $P$. Then the results of
\cite{DBLP:conf/lpnmr/JanhunenNS09} imply that there is a model
$\pair{M}{\val}$ of $\tr{DL}{P}$ such that $S=M\isect\hb{P}$.  The
valuation $\val$ is condensed for each non-trivial SCC
$S$ of $\DG{P}$ as follows. Let us partition $S$ into
$\rg{S_0}{\dunion}{S_n}$ such that
(i)
$\val(x_{a})=\val(x_{b})$
for each $0\leq i\leq n$ and $a,b\in S_i$,
(ii)
$\val(x_a)=\val(z)$%
\footnote{A special variable $z$ is used as a placeholder for
the constant $0$ in the translation $\tr{DL}{P}$
\cite{DBLP:conf/lpnmr/JanhunenNS09}.}
for each $a\in S_0$, and
(iii)
for each $0\leq i<j\leq n$, $a\in S_i$, and $b\in S_j$,
$\val(x_{a})\leq\val(x_{b})$.
Then define $\val'$ for the bit vector $x_a$ associated with an atom
$a\in S_i$ by setting $\val'(x_a,j)=1$ iff the $\ord{j}$ bit of
$\bin{i}$ is $1$, i.e., $\val'(x_a)=i$. It follows that
$\pair{I}{\val}\models x_b\leq x_a-1$
iff
$\pair{I}{\val'}\models\bvlt{x_b}{m}{x_a}$
for any $a,b\in S$. Moreover, we have
$\pair{M}{\val}\models (x_a\leq z+0)\land(z\leq x_a+0)$
iff
$\pair{M}{\val'}\models\bveq{x_a}{m}{\bin{0}}$
for any $a\in S$. Due to the similar structures of $\tr{DL}{P}$ and
$\tr{BV}{P}$, we obtain $\pair{M}{\val}\models\tr{BV}{P}$ as desired.

($\impliedby$)
Let $\pair{M}{\val}$ be a model of $\tr{BV}{P}$. Then define $\val'$
such that
$\val'(x)=\sum_{i=1}^{\width{x}}(\val(x,i)\cdot2^{\width{x}-i})$
where $x$ on the left hand side stands for the integer variable
corresponding to the bit vector $x$ on the right hand side. It follows
that
$\pair{I}{\val}\models\bvlt{x_b}{m}{x_a}$
iff
$\pair{I}{\val'}\models x_b\leq x_a-1$.
By setting $\val'(z)=0$, we obtain
$\pair{M}{\val}\models\bveq{x_a}{m}{\bin{0}}$
if and only if
$\pair{M}{\val'}\models (x_a\leq z+0)\land(z\leq x_a+0)$.
The strong analogy present in the structures of $\tr{BV}{P}$ and
$\tr{DL}{P}$ implies that $\pair{M}{\val'}$ is a model of
$\tr{DL}{P}$.  Thus, $S=M\isect\hb{P}$ is an answer set of $P$ by
\cite{DBLP:conf/lpnmr/JanhunenNS09}.
\qed
\end{proof}

Even tighter relationships of answer sets and models can be
established for the translations $\tr{BV^{l}}{P}$, $\tr{BV^{g}}{P}$,
and $\tr{BV^{lg}}{P}$. It can be shown that the model $\pair{M}{\val}$
of $\tr{BV^*}{P}$ corresponding to an answer set $S$ of $P$ is
unique, i.e., there is no other model $\pair{N}{\val'}$ of the
translation such that $S=N\isect\hb{P}$.
These results contrast with \cite{DBLP:conf/lpnmr/JanhunenNS09}:
the analogous extensions $\tr{DL^*}{P}$ guarantee the uniqueness
of $M$ in a model $\pair{M}{\val}$ but there are always infinitely
many copies $\pair{M}{\val'}$ of $\pair{M}{\val}$ such that
$\pair{M}{\val'}\models\tr{DL^*}{P}$.
Such a valuation $\val'$ can be simply obtained by setting
$\val'(x)=\val(x)+1$ for any $x$.

%------------------------------------------------------------------------------

\section{Native Support for Extended Rule Types}
\label{section:native-extensions}

The input syntax of the \system{smodels} system was soon
extended by further rule types
\cite{DBLP:journals/ai/SimonsNS02}.
In solver interfaces, the rule types usually
take the following simple syntactic forms:
\begin{eqnarray}
\label{eq:choice-rule}
\choice{\rg{a_1}{,}{a_l}}\IF\rg{b_1}{,}{b_n},\naf\rg{c_1}{,}{\naf c_m}\END \\
\label{eq:cardinality-rule}
a\IF\limit{l}{\rg{b_1}{,}{b_n},\naf\rg{c_1}{,}{\naf c_m}}\END \\
\label{eq:weight-rule}
a\IF\limit{l}{\rg{b_1=w_{b_1}}{,}{b_n=w_{b_n}},
              \rg{\naf c_1=w_{c_1}}{,}{\naf c_m=w_{c_m}}}\END
\end{eqnarray}
The body of a \emph{choice rule} (\ref{eq:choice-rule}) is interpreted
in the same way as that of a normal rule~(\ref{eq:rule}). The head, in
contrast, allows to derive any subset of atoms $\rg{a_1}{,}{a_l}$, if
the body is satisfied, and to make a \emph{choice} in this way.
The head $a$ of a \emph{cardinality rule} (\ref{eq:cardinality-rule})
is derived, if its body is satisfied, i.e., the number of satisfied
literals amongst $\rg{b_1}{,}{b_n}$ and $\rg{\naf c_1}{,}{\naf c_m}$
is at least $l$ acting as the \emph{lower bound}.
A \emph{weight rule} of the form (\ref{eq:weight-rule}) generalizes
this idea by assigning arbitrary positive weights to literals (rather
than 1s). The body is satisfied if the sum of weights assigned to
satisfied literals is at least $l$, thus enabling one to infer the
head $a$ using the rule.
In practise, the grounding components used in ASP systems allow for more
versatile use of cardinality and weight rules, but the primitive forms
(\ref{eq:choice-rule}), (\ref{eq:cardinality-rule}), and
(\ref{eq:weight-rule}) provide a solid basis for efficient
implementation via translations. The reader is referred to
\cite{DBLP:journals/ai/SimonsNS02}
for a generalization of answer sets for programs involving
such extended rule types. The respective class of
\emph{weight constraint programs} (WCPs)
is typically supported by \system{smodels} compatible systems.

Whenever appropriate, it is possible to translate extended rule types
as introduced above back to normal rules. To this end, a number of
transformations are addressed in~\cite{JN11:mg65} and they have been
implemented as a tool called \system{lp2normal}%
\footnote{\label{fn:asptools}\url{http://www.tcs.hut.fi/Software/asptools/}}.
For instance, the head of a choice rule (\ref{eq:choice-rule}) can be
captured in terms of rules
\begin{center}
\begin{tabular}{lcccl}
$a_1\IF b,\naf \compl{a_1}\END$ && \ldots && $a_l\IF b,\naf \compl{a_l}\END$ \\
$\compl{a_1}\IF\naf a_1\END$ && \ldots && $\compl{a_l}\IF\naf a_l\END$ \\
\end{tabular}
\end{center}
where $\rg{\compl{a_1}}{,}{\compl{a_l}}$ are new atoms and $b$ is a
new atom standing for the body of (\ref{eq:choice-rule}) which can be
defined using (\ref{eq:choice-rule}) with the head replaced by $b$.
We assume that this transformation is applied at first to remove
choice rules when the goal is to translate extended rule types into
bit-vector logic. The strength of this transformation is locality,
i.e., it can be applied on a rule-by-rule basis, and linearity with
respect to the length of the original rule (\ref{eq:choice-rule}).  To
the contrary, linear normalization of cardinality and weight rules
seems impossible. Thus, we also provide direct translations into formulas
of bit-vector logic.

We present the translation of a weight rule (\ref{eq:weight-rule})
whereas the translation of a cardinality rule
(\ref{eq:cardinality-rule}) is obtained as a special case
$\rg{w_{b_1}}{=}{w_{b_n}}=\rg{w_{c_1}}{=}{w_{c_m}}=1$. The body of a
weight rule can be evaluated using bit vectors $\rg{s_1}{,}{s_{n+m}}$
of width
$k=\lceil\log_2(\sum_{i=1}^{n}w_{b_i}+\sum_{i=1}^{m}w_{c_i}+1)\rceil$
constrained by $2\times(n+m)$ formulas
\begin{center}
\begin{tabular}{lccl}
$b_1\limpl(\bveq{s_1}{k}{\bin{w_{b_1}}})$, &&&
$\neg b_1\limpl(\bveq{s_1}{k}{\bin{0}})$, \\
$b_2\limpl(\bveq{s_2}{k}{\bvadd{s_1}{k}{\bin{w_{b_2}}}})$, &&&
$\neg b_2\limpl(\bveq{s_2}{k}{s_1})$, \\
\vdots &&& \vdots \\
$b_n\limpl(\bveq{s_n}{k}{\bvadd{s_{n-1}}{k}{\bin{w_{b_n}}}})$, &&&
$\neg b_n\limpl(\bveq{s_n}{k}{s_{n-1}})$, \\
$c_1\limpl(\bveq{s_{n+1}}{k}{s_n})$, &&&
$\neg c_1\limpl(\bveq{s_{n+1}}{k}{\bvadd{s_n}{k}{\bin{w_{c_1}}}})$, \\
\vdots &&& \vdots \\
$c_m\limpl(\bveq{s_{n+m}}{k}{s_{n+m-1}})$, &&&
$\neg c_m\limpl(\bveq{s_{n+m}}{k}{\bvadd{s_{n+m-1}}{k}{\bin{w_{c_m}}}})$. \\
\end{tabular}
\end{center}
The lower bound $l$ of (\ref{eq:weight-rule}) can be checked in terms
of the formula $\neg(\bvlt{s_{n+m}}{k}{\bin{l}})$ where we assume that
$\bin{l}$ is of width $k$, since the rule can be safely deleted
otherwise. In view of the overall translation, the formula
$\bdatom{r}\lequiv\neg(\bvlt{s_{n+m}}{k}{\bin{l}})$ can be used in
conjunction with the completion formula
(\ref{eq:completion-disjunction}).
Weight rules also contribute to the dependency graph $\DG{P}$ in
analogy to normal rules, i.e., the head $a$ depends on all positive
body atoms $\rg{b_1}{,}{b_n}$. In this way, $\tr{BV}{P}$
generalizes for programs $P$ having extended rules.

%------------------------------------------------------------------------------
\section{Experimental Results}
\label{section:experiments}

\begin{figure}[t]
\begin{center}
\begin{minipage}[t]{0.6\textwidth}
\begin{verbatim}
gringo program.lp instance.lp \
| smodels -internal -nolookahead \
| lpcat -s=symbols.txt \
| lp2bv [-l] [-g] \
| boolector -fm
\end{verbatim}
\end{minipage}
\end{center}
\caption{Unix shell pipeline for running a benchmark instance
\label{fig:script}}
\end{figure}

A new translator called \system{lp2bv} was implemented as a derivative of
\system{lp2diff}\footnote{\url{http://www.tcs.hut.fi/Software/lp2diff/}}
that translates logic programs into difference logic. In contrast,
the new translator will provide its output in the bit-vector format. 
In analogy to its predecessor, it expects to receive its input in the
\system{smodels}\footnote{\url{http://www.tcs.hut.fi/Software/smodels/}} 
file format. Models of the resulting bit-vector theory are searched for using
\system{boolector}\footnote{\url{http://fmv.jku.at/boolector/}}
(v.~1.4.1) 
\cite{DBLP:conf/tacas/BrummayerB09} and
\system{z3}%
\footnote{\url{http://research.microsoft.com/en-us/um/redmond/projects/z3/}} 
(v.~2.11) \cite{DBLP:conf/tacas/MouraB08} as back-end solvers. 
The goal of our preliminary experiments was to see how the
performances of systems based on \system{lp2bv} compare with the
performance of a state-of-the-art ASP solver
\system{clasp}\footnote{\url{http://www.cs.uni-potsdam.de/clasp/}}
(v.~1.3.5) \cite{DBLP:conf/lpnmr/GebserKNS07a}.
The experiments were based on the NP-complete benchmarks of the ASP
Competition 2009. In this benchmark collection, there are 23 benchmark
problems with 516 instances in total.
Before invoking a translator and the respective SMT solver, we
performed a few preprocessing steps, as detailed in Figure
\ref{fig:script}, by calling:

\begin{itemize}
\item
\system{gringo} (v.~2.0.5), for grounding 
the problem encoding and a given instance;

\item
\system{smodels}\footnote{\url{http://www.tcs.hut.fi/Software/smodels/}} 
(v.~2.34), for simplifying the resulting ground program;

\item
\system{lpcat} (v.~1.18), for removing all unused atom numbers,
for making the atom table of the ground program contiguous, and
for extracting the symbols for later use; and

\item
\system{lp2normal} (version 1.11), for normalizing the program.
\end{itemize}
The last step is optional and not included as part of the pipeline
in Figure \ref{fig:script}. Pipelines of this kind were executed
under Linux/Ubuntu operating system running on six-core
AMD Opteron$^{\mathrm{(TM)}}$ 2435 processors under 2.6 GHz clock
rate and with 2.7 GB memory limit that corresponds to the amount
of memory available in the ASP Competition 2009.

For each system based on a translator and a back-end solver, there are
four variants of the system to consider: W indicates that only weak
ranking constraints are used, while L, G, and LG mean that either
local, or global, or both local and global strong ranking constraints,
respectively, are employed when translating the logic program.

\begin{table}[t]
\begin{center}
\caption{Experimental results without normalization}
\label{expRes}
{\fontsize{6}{9}\selectfont
\centering
\begin{tabular}{|p{3.1cm}|c||c||c|c|c|c|c|c|c|c|c|c|c|c|}
\hline
& INST & \textsc{CLASP} & \multicolumn{4}{c|}{\textsc{LP2BV+BOOLECTOR}} & \multicolumn{4}{c|}{\textsc{LP2BV+Z3}} & \multicolumn{4}{c|}{\textsc{LP2DIFF+Z3}} \\
Benchmark &&  & W & L & G & LG & W & L & G & LG & W & L & G & LG \\
\hline \hline
        Overall Performance  	& 516  &  465  &  276  &  244  &  261  &  256  &  217  &  216  &  194  &  204  &\hl{360}&  349  &  324  &  324 \\
								&	   &347/118&188/ 88&161/ 83&174/ 87&176/ 80&142/ 75&147/ 69&124/ 70&135/ 69&\hl{257/103}&251/ 98&225/ 99&226/ 98 \\
\hline\hline
                KnightTour      &  10  &  8/ 0 &  2/ 0 &  1/ 0 &  0/ 0 &  0/ 0 &  1/ 0 &  0/ 0 &  0/ 0 &  1/ 0 &\hl{6/ 0}&\hl{6/ 0}&  4/ 0 &  5/ 0 \\
            GraphColouring      &  29  &  8/ 0 &  \hl{7/0} &  \hl{7/0} &  \hl{7/0} &  \hl{7/0} &  6/ 0 &  \hl{7/0} &  \hl{7/0} &  \hl{7/0} &  \hl{7/0} &  \hl{7/0} &  \hl{7/0} &  \hl{7/0} \\
               WireRouting      &  23  & 11/11 &  2/ 3 &  1/ 1 &  1/ 2 &  0/ 2 &  1/ 3 &  0/ 0 &  0/ 0 &  0/ 1 &  3/ 3 &  2/ 3 &  2/ 4 &  \hl{5/3} \\
     DisjunctiveScheduling      &  10  &  5/ 0 &  0/ 0 &  0/ 0 &  0/ 0 &  0/ 0 &  0/ 0 &  0/ 0 &  0/ 0 &  0/ 0 &  0/ 0 &  0/ 0 &  0/ 0 &  0/ 0 \\
         GraphPartitioning      &  13  &  6/ 7 &  3/ 0 &  3/ 0 &  3/ 0 &  3/ 0 &  4/ 0 &  4/ 0 &  4/ 0 &  3/ 0 &  \hl{6/2} &  6/ 1 &  6/ 1 &  6/ 1 \\
            ChannelRouting      &  11  &  6/ 2 &  \hl{6/2} &  \hl{6/2} &  \hl{6/2} &  \hl{6/2} &  5/ 2 &  \hl{6/2} &  \hl{6/2} &  \hl{6/2} &  \hl{6/2} &  \hl{6/2} &  \hl{6/2} &  \hl{6/2} \\
                 Solitaire      &  27  & 19/ 0 &  2/ 0 &  5/ 0 &  1/ 0 &  4/ 0 &  0/ 0 &  0/ 0 &  0/ 0 &  0/ 0 & \hl{21/0} & \hl{21/0} & 20/ 0 & \hl{21/0} \\
                 Labyrinth      &  29  & 26/ 0 &  \hl{1/0} &  0/ 0 &  0/ 0 &  0/ 0 &  0/ 0 &  0/ 0 &  0/ 0 &  0/ 0 &  0/ 0 &  0/ 0 &  0/ 0 &  0/ 0 \\
WeightBoundedDominatingSet      &  29  & 26/ 0 & 18/ 0 & 18/ 0 & 17/ 0 & 18/ 0 & 12/ 0 & 12/ 0 & 11/ 0 & 12/ 0 & \hl{22/0} & \hl{22/0} & \hl{22/0} & 21/ 0 \\
            MazeGeneration      &  29  & 10/15 &  8/15 &  1/15 &  0/15 &  0/16 &  5/16 &  1/15 &  0/15 &  1/15 & \hl{\hl{10/17}} & 10/15 &  5/15 &  4/15 \\
                  15Puzzle      &  16  & 16/ 0 & \hl{16/0} & 15/ 0 & 14/ 0 & 15/ 0 &  4/ 0 &  4/ 0 &  5/ 0 &  5/ 0 &  0/ 0 &  0/ 0 &  0/ 0 &  0/ 0 \\
            BlockedNQueens      &  29  & 15/14 &  2/ 2 &  0/ 2 &  1/ 2 &  0/ 2 &  1/ 0 &  2/ 0 &  2/ 0 &  0/ 0 & \hl{15/13} & \hl{15/13} & 15/12 & \hl{15/13} \\
    ConnectedDominatingSet      &  21  & 10/10 & \hl{10/11} &  9/ 8 & \hl{10/11} &  6/ 3 & 10/10 &  9/10 & 10/ 9 & 10/ 9 &  9/ 8 &  7/ 6 &  9/ 7 &  7/ 6 \\
              EdgeMatching      &  29  & 29/ 0 &  0/ 0 &  0/ 0 &  0/ 0 &  0/ 0 &  0/ 0 &  0/ 0 &  0/ 0 &  0/ 0 &  \hl{3/0} &  1/ 0 &  \hl{3/0} &  2/ 0 \\
                  Fastfood      &  29  & 10/19 &  9/16 & 10/16 & 10/16 &  9/16 &  9/ 9 &  9/ 9 &  9/10 &  9/ 9 & \hl{10/18} & \hl{10/18} & \hl{10/18} & \hl{10/18} \\
    GeneralizedSlitherlink      &  29  & 29/ 0 & \hl{29/0} & 20/ 0 & \hl{29/0} & \hl{29/0} & \hl{29/0} & \hl{29/0} & 16/ 0 & \hl{29/0} & \hl{29/0} & \hl{29/0} & \hl{29/0} & \hl{29/0} \\
           HamiltonianPath      &  29  & 29/ 0 & 27/ 0 & 25/ 0 & \hl{29/0} & 28/ 0 & 26/ 0 & 27/ 0 & 25/ 0 & 26/ 0 & \hl{29/0} & \hl{29/0} & \hl{29/0} & \hl{29/0} \\
                     Hanoi      &  15  & 15/ 0 & \hl{15/0} & \hl{15/0} & \hl{15/0} & \hl{15/0} &  5/ 0 &  5/ 0 &  5/ 0 &  4/ 0 & \hl{15/0} & \hl{15/0} & \hl{15/0} & \hl{15/0} \\
    HierarchicalClustering      &  12  &  8/ 4 &  \hl{8/4} &  \hl{8/4} &  \hl{8/4} &  \hl{8/4} &  4/ 4 &  4/ 4 &  4/ 4 &  4/ 4 &  \hl{8/4} &  \hl{8/4} &  \hl{8/4} &  \hl{8/4} \\
              SchurNumbers      &  29  & 13/16 &  6/16 &  5/16 &  5/16 &  5/16 &  9/16 &  9/16 &  9/16 &  9/16 & \hl{11/16} & \hl{11/16} & \hl{11/16} & \hl{11/16} \\
                   Sokoban      &  29  &  9/20 &  9/19 &  8/19 &  8/19 &  8/19 &  7/15 &  7/13 &  7/14 &  5/13 &  \hl{9/20} &  \hl{9/20} &  \hl{9/20} &  \hl{9/20} \\
                    Sudoku      &  10  & 10/ 0 &  5/ 0 &  4/ 0 &  4/ 0 &  5/ 0 &  4/ 0 &  4/ 0 &  4/ 0 &  4/ 0 &  \hl{9/0} &  8/ 0 &  8/ 0 &  \hl{9/0} \\
     TravellingSalesperson      &  29  & 29/ 0 &  3/ 0 &  0/ 0 &  6/ 0 & 10/ 0 &  0/ 0 &  8/ 0 &  0/ 0 &  0/ 0 & \hl{29/0} & \hl{29/0} &  7/ 0 &  7/ 0 \\

\hline
\end{tabular}
}
\end{center}
\end{table}

Table \ref{expRes} collects the results from our experiments without
normalization whereas Table \ref{expResNor} shows the results when
\system{lp2normal} \cite{JN11:mg65} was used to remove extended rule
types discussed in Section \ref{section:native-extensions}.  In both
tables, the first column gives the name of the benchmark, followed by the
number of instances of that particular benchmark in the second column.
The following columns indicate the numbers of instances that were
solved by the systems considered in our experiments.  A notation like
8/4 means that the system was able to solve eight satisfiable and four
unsatisfiable instances in that particular benchmark.
Hence, if there are 15 instances in a benchmark and the system could
only solve 8/4, this means that the system was unable to solve the
remaining three instances within the time limit of 600 seconds,
i.e. ten minutes, per instance%
\footnote{One observation is that the performance of systems based on
  \system{lp2bv} is quite stable: even when we extended the
  time limit to 20 minutes, the results did not change much
  (differences of only one or two instances were perceived in most cases).}.
As regards the number of solved instances in each benchmark, the
best performing translation-based approaches are highlighted in boldface.
Though it was not shown in all tables, we also run the experiments
using translator \system{lp2diff} with \system{z3} as back-end solver,
and the summary is included in Table \ref{expSum}---giving an overview
of experimental results in terms of total numbers of instances
solved out of 516.

It is apparent that the systems based on \system{lp2bv} did not
perform very well without normalization. As indicated by Table
\ref{expSum}, the overall performance was even worse than that of
systems using \system{lp2diff} for translation and \system{z3} for
model search. However, if the input was first translated into a
normal logic program using \system{lp2normal}, i.e., before
translation into a bit-vector theory, the performance was clearly
better. Actually, it exceeded that of the systems based on
\system{lp2diff} and became closer to that of \system{clasp}.
We note that normalization does not help so much in case of
\system{lp2diff} and the experimental results obtained using both
normalized and unnormalized instances are quite similar in terms of
solved instances. Thus it seems that solvers for bit-vector logic are
not able to make the best of native translations of cardinality and
weight rules from
Section~\ref{section:native-extensions}
in full. If an analogous translation into difference logic is used, as
implemented in \system{lp2diff}, such a negative effect was not
perceived using \system{z3}. Our understanding is that the efficient
graph-theoretic satisfiability check for difference constraints used
in the search procedure of \system{z3} turns the native translation
feasible as well.
As indicated by our test results, \system{boolector} is clearly better
back-end solver for \system{lp2bv} than \system{z3}. This was to be
expected since \system{boolector} is a native solver for bit-vector
logic whereas \system{z3} supports a wider variety of SMT fragments
and can be used for more general purposes. Moreover, the design of
\system{lp2bv} takes into account operators of bit-vector logic which
are directly supported by \system{boolector} and not implemented as
syntactic sugar.

\begin{table}[t]
\begin{center}
\caption{Experimental results with normalization}
\label{expResNor}
{\fontsize{7}{9}\selectfont
\centering
\begin{tabular}{|p{3.1cm}|c||c||c|c|c|c|c|c|c|c|}
\hline
& INST & \textsc{CLASP} & \multicolumn{4}{c|}{\textsc{LP2BV+BOOLECTOR}} & \multicolumn{4}{c|}{\textsc{LP2BV+Z3}}\\
Benchmark & & & W & L & G & LG & W & L & G & LG\\
\hline \hline
        Overall Performance  & 516  &  459&\hl{381}&  343&  379&\hl{381}&  346&  330&  325&  331 \\
                            &	   &346/113&\hl{279/102}&243/100&278/101&\hl{281/100}&240/106&231/ 99&224/101&232/ 99 \\
\hline\hline
                KnightTour  &  10  & 10/ 0 &  \hl{2/0} &  \hl{2/0} &  1/ 0 &  0/ 0 &  1/ 0 &  0/ 0 &  0/ 0 &  0/ 0 \\
            GraphColouring  &  29  &  9/ 0 &  8/ 0 &  8/ 0 &  8/ 0 &  8/ 0 &  \hl{9/2} &  \hl{9/2} &  \hl{9/2} &  \hl{9/2} \\
               WireRouting  &  23  & 11/11 &  2/ 6 &  1/ 3 &  1/ 3 &  1/ 3 &  \hl{2/7} &  1/ 4 &  1/ 4 &  1/ 3 \\
     DisjunctiveScheduling  &  10  &  5/ 0 &  \hl{5/0} &  \hl{5/0} &  \hl{5/0} &  \hl{5/0} &  \hl{5/0} &  \hl{5/0} &  \hl{5/0} &  \hl{5/0} \\
         GraphPartitioning  &  13  &  4/ 1 &  \hl{5/0} &  \hl{5/0} &  4/ 0 &  \hl{5/0} &  2/ 1 &  2/ 1 &  2/ 1 &  2/ 0 \\
            ChannelRouting  &  11  &  6/ 2 &  \hl{6/2} &  \hl{6/2} &  \hl{6/2} &  \hl{6/2} &  \hl{6/2} &  \hl{6/2} &  \hl{6/2} &  \hl{6/2} \\
                 Solitaire  &  27  & 18/ 0 & \hl{23/0} & \hl{23/0} & \hl{23/0} & \hl{23/0} & 22/ 0 & 22/ 0 & 22/ 0 & 22/ 0 \\
                 Labyrinth  &  29  & 27/ 0 &  1/ 0 &  1/ 0 &  2/ 0 &  \hl{3/0} &  0/ 0 &  0/ 0 &  0/ 0 &  0/ 0 \\
WeightBoundedDominatingSet  &  29  & 25/ 0 & 15/ 0 & 15/ 0 & 15/ 0 & \hl{16/0} & 10/ 0 & 10/ 0 & 10/ 0 & 10/ 0 \\
            MazeGeneration  &  29  & 10/15 &  \hl{8/15} &  0/15 &  0/15 &  0/16 &  5/16 &  0/15 &  0/15 &  0/15 \\
                  15Puzzle  &  16  & 15/ 0 & \hl{16/0} & \hl{16/0} & \hl{16/0} & \hl{16/0} & 11/ 0 & 10/ 0 & 11/ 0 & 11/ 0 \\
            BlockedNQueens  &  29  & 15/14 & 14/14 & 14/14 & 14/14 & 14/14 & \hl{15/14} & \hl{15/14} & \hl{15/14} & \hl{15/14} \\
    ConnectedDominatingSet  &  21  & 10/11 & \hl{10/11} &  8/11 &  9/11 &  9/10 & \hl{10/11} &  9/11 &  9/11 &  9/11 \\
              EdgeMatching  &  29  & 29/ 0 & 29/ 0 & 29/ 0 & 29/ 0 & 29/ 0 & 29/ 0 & 29/ 0 & 29/ 0 & 29/ 0 \\
                  Fastfood  &  29  & 10/19 &  9/14 &  9/15 &  \hl{9/16} &  9/15 &  0/13 &  0/10 &  0/12 &  0/12 \\
    GeneralizedSlitherlink  &  29  & 29/ 0 & \hl{29/ 0} & 21/ 0 & \hl{29/ 0} & \hl{29/ 0} & \hl{29/ 0} & \hl{29/ 0} & 21/ 0 & \hl{29/ 0} \\
           HamiltonianPath  &  29  & 29/ 0 & \hl{29/ 0} & 28/ 0 & \hl{29/ 0} & \hl{29/ 0} & \hl{29/ 0} & \hl{29/ 0} & \hl{29/ 0} & \hl{29/ 0} \\
                     Hanoi  &  15  & 15/ 0 & \hl{15/ 0} & \hl{15/ 0} & \hl{15/ 0} & \hl{15/ 0} & \hl{15/ 0} & \hl{15/ 0} & \hl{15/ 0} & \hl{15/ 0} \\
    HierarchicalClustering  &  12  &  8/ 4 &  \hl{8/4} &  \hl{8/4} &  \hl{8/4} &  \hl{8/4} &  \hl{8/4} &  \hl{8/4} &  \hl{8/4} &  \hl{8/4} \\
              SchurNumbers  &  29  & 13/16 & 10/16 & 10/16 &  9/16 & 10/16 & \hl{13/16} & \hl{13/16} & \hl{13/16} & \hl{13/16} \\
                   Sokoban  &  29  &  9/20 &  \hl{9/20} &  \hl{9/20} &  \hl{9/20} &  \hl{9/20} &  \hl{9/20} &  \hl{9/20} &  \hl{9/20} &  \hl{9/20} \\
                    Sudoku  &  10  & 10/ 0 & \hl{10/0} & \hl{10/0} & \hl{10/0} & \hl{10/0} & \hl{10/0} & \hl{10/0} & \hl{10/0} & \hl{10/0} \\
     TravellingSalesperson  &  29  & 29/ 0 & 16/ 0 &  0/ 0 & \hl{27/0} & \hl{27/0} &  0/ 0 &  0/ 0 &  0/ 0 &  0/ 0 \\
\hline
\end{tabular}}
\end{center}
\end{table}

In addition, we note on the basis of our results that the performance of
the state-of-the-art ASP solver \system{clasp} is significantly better,
and the translation-based approaches to computing stable models are
still left behind. By the results of Table \ref{expResNor}, even the best
variants of systems based on \system{lp2bv} did not work well enough to
compete with \system{clasp}. The difference is especially due to the
following benchmarks:
\emph{Knight Tour},
\emph{Wire Routing},
\emph{Graph Partitioning},
\emph{Labyrinth},
\emph{Weight Bounded Dominating Set},
\emph{Fastfood}, and
\emph{Travelling Salesperson}.
All of them involve either recursive rules
(\emph{Knight Tour}, \emph{Wire Routing}, and \emph{Labyrinth}),
weight rules (\emph{Weight Bounded Dominating Set} and \emph{Fastfood}),
or both (\emph{Graph Partitioning} and \emph{Travelling Salesperson}).
Hence, it seems that handling recursive rules and weight constraints
in the translational approach is less efficient compared to their
native implementation in \system{clasp}.
When using the current normalization techniques to remove cardinality
and weight rules, the sizes of ground programs tend to increase
significantly and, in particular, if weight rules are abundant.  For
example, after normalization the ground programs are ten times larger
for the benchmark \emph{Weight Bounded Dominating Set}, and five times
larger for \emph{Fastfood}.
It is also worth pointing out that the efficiency of \system{clasp}
turned out to be insensitive to normalization.

While having trouble with recursive rules and weight constraints for
particular benchmarks, the translational approach handles certain large
instances quite well. The largest instances in the experiments belong
to the \emph{Disjunctive Scheduling} benchmark, of which all instances
are ground programs of size over one megabyte but after normalization%
\footnote{In this benchmark, normalization does not affect
the size of grounded programs significantly.},
the \system{lp2bv} systems can solve as many instances as \system{clasp}.

\begin{table}[t]
\begin{center}
\caption{Summary of the experimental results}
\label{expSum}
\centering
\begin{tabular}{|ll|ccc|ccc|ccc|ccc|}
\hline
		System			 		&&&  W  &&&  L  &&&  G  &&& LG & \\
\hline \hline
\textsc{LP2BV+BOOLECTOR} 		&\quad\quad&\quad& 276 &\quad&\quad& 244 &\quad&\quad& 261 &\quad&\quad& 256 &\quad\\
\textsc{LP2BV+Z3}		 		&\quad\quad&\quad& 217 &\quad&\quad& 216 &\quad&\quad& 194 &\quad&\quad& 204 &\quad\\
\textsc{LP2DIFF+Z3}		 		&\quad\quad&\quad& 360 &\quad&\quad& 349 &\quad&\quad& 324 &\quad&\quad& 324 &\quad\\
\hline \hline
\textsc{CLASP}					&\quad& \multicolumn{12}{c|}{465} \\
\hline \hline
\textsc{LP2NORMAL2BV+BOOLECTOR} &\quad\quad&\quad& 381 &\quad&\quad& 343 &\quad&\quad& 379 &\quad&\quad& 381 &\quad\\
\textsc{LP2NORMAL2BV+Z3} 		&\quad\quad&\quad& 346 &\quad&\quad& 330 &\quad&\quad& 325 &\quad&\quad& 331 &\quad\\
\textsc{LP2NORMAL2DIFF+Z3} 		&\quad\quad&\quad& 364 &\quad&\quad& 357 &\quad&\quad& 349 &\quad&\quad& 349 &\quad\\
\hline \hline
\textsc{LP2NORMAL+CLASP}					&\quad& \multicolumn{12}{c|}{459} \\
\hline
\end{tabular}
\end{center}
\end{table}

%------------------------------------------------------------------------------

\section{Conclusion}
\label{section:conclusion}

In this paper, we present a novel and concise translation from normal
logic programs into fixed-width bit-vector theories. Moreover, the
extended rule types supported by \system{smodels} compatible answer
set solvers can be covered via native translations. The length of the
resulting translation is linear with respect to the length of the
original program. The translation has been implemented as a
translator, \system{lp2bv}, which enables the use of bit-vector
solvers in the search for answer sets. Our preliminary experimental
results indicate a level of performance which is similar to that
obtained using solvers for difference logic. However, this presumes
one first to translate extended rule types into normal rules and then to
apply the translation into bit-vector logic. One potential
explanation for such behavior is the way in which SMT solvers
implement reasoning with bit vectors: a predominant strategy is to
translate theory atoms involving bit vectors into propositional
formulas and to apply satisfiability checking techniques
systematically.  We anticipate that an improved performance could be
obtained if a native support for certain bit vector primitives were
incorporated into SMT solvers directly.
When comparing to the state-of-the-art ASP solver \system{clasp}, 
we noticed that the performance of the translation based approach
compared unfavorably, in particular, for benchmarks which contained
recursive rules or weight constraints or both. This indicates that
the performance can be improved by developing new translation 
techniques for these two features. 
In order to obtain a more comprehensive view of the performance
characteristics of the translational approach, the plan is to extend our
experimental setup to include benchmarks that were used in the third ASP
competition~\cite{DBLP:conf/lpnmr/CalimeriIRABCCFFLMMPPRSSTV11}. Moreover,
we intend to use the new SMT library format \cite{SMT-LIB} in future
versions of our translators.

\paragraph{Acknowledgments}
This research has been partially funded by the Academy of Finland
under the project
``\emph{Methods for Constructing and Solving Large Constraint Models}''
(MCM, \#122399).

%------------------------------------------------------------------------------

%==============================================================================

\end{document}